\documentclass[pdflatex,sn-mathphys-num]{sn-jnl}

\usepackage{graphicx}%
\usepackage{multirow}%
\usepackage{amsmath,amssymb,amsfonts}%
\usepackage{amsthm}%
\usepackage{mathrsfs}%
\usepackage[title]{appendix}%
\usepackage{textcomp}%
\usepackage{manyfoot}%
\usepackage{booktabs}%
\usepackage{algorithm}%
\usepackage{algpseudocode}%
\usepackage{listings}%

\usepackage{bm}
\usepackage{xspace}
\usepackage[table,dvipsnames]{xcolor}
\usepackage{url}

\newcommand\BibTeX{{\rmfamily B\kern-.05em \textsc{i\kern-.025em b}\kern-.08em
T\kern-.1667em\lower.7ex\hbox{E}\kern-.125emX}}

\newcommand{\statespace}{\ensuremath{\mathbf X}\xspace}
\newcommand{\obsspace}{\ensuremath{\mathbf Y}\xspace}
\newcommand{\actspace}{\ensuremath{\mathbf U}\xspace}
\newcommand{\imgcontroller}{\ensuremath{\emph{h}}\xspace}
\newcommand{\dynmodel}{\ensuremath{f}\xspace}
\newcommand{\obsmodel}{\ensuremath{o}\xspace}
\newcommand{\initstate}{\ensuremath{x_0}\xspace}
\newcommand{\safeprop}{\ensuremath{\varphi}\xspace}

\newcommand{\obsseqspace}{\ensuremath{\obsspace^m}\xspace}
\newcommand{\dataset}{\ensuremath{\mathbf{Z}}\xspace}
\newcommand{\dataseq}{\ensuremath{\mathbf{z}}\xspace}
\newcommand{\labelpred}{\ensuremath{\rho}\xspace}
\newcommand{\chancepred}{\ensuremath{g}\xspace}

\newcommand{\imgforecaster}{\ensuremath{f_{g}}\xspace}
\newcommand{\latforecaster}{\ensuremath{f_{l}}\xspace}
\newcommand{\enc}{\ensuremath{e}\xspace}
\newcommand{\dec}{\ensuremath{d}\xspace}

\theoremstyle{thmstyleone}%
\newtheorem{theorem}{Theorem}%
\theoremstyle{thmstylethree}%
\newtheorem{definition}{Definition}%
\newtheorem{problem}{Problem}%
\newtheorem*{problem*}{Problem}%

\begin{document}

\title[How Safe Will I Be Given What I Saw?]{How Safe Will I Be Given What I Saw? Calibrated Safety Prediction for Image-Controlled Autonomy}

\author*[1]{\fnm{Zhenjiang} \sur{Mao}}\email{z.mao@ufl.edu}
\author[1]{\fnm{Mrinall Eashaan} \sur{Umasudhan}}
\author[1]{\fnm{Ivan} \sur{Ruchkin}}

\affil[1]{\orgdiv{Trustworthy Engineered Autonomy (TEA) Lab}, \orgname{University of Florida}, \orgaddress{\city{Gainesville}, \state{FL}, \country{USA}}}

\abstract{Autonomous robots that rely on deep neural network controllers pose critical challenges for safety prediction, especially under partial observability and distribution shift. Traditional model-based verification techniques are limited in scalability and require access to low-dimensional state models, while model-free methods often lack reliability guarantees. This paper addresses these limitations by introducing a framework for calibrated safety prediction in end-to-end vision-controlled systems, where neither the state-transition model nor the observation model is accessible. Building on the foundation of world models, we leverage variational autoencoders and recurrent predictors to forecast future latent trajectories from raw image sequences and estimate the probability of satisfying safety properties. We distinguish between monolithic and composite prediction pipelines and introduce a calibration mechanism to quantify prediction confidence. In long-horizon predictions from high-dimensional observations, the forecasted inputs to the safety evaluator can deviate significantly from the training distribution due to compounding prediction errors and changing environmental conditions, leading to miscalibrated risk estimates. To address this, we incorporate unsupervised domain adaptation to ensure robustness of safety evaluation under distribution shift in predictions without requiring manual labels. Our formulation provides theoretical calibration guarantees and supports practical evaluation across long prediction horizons. Experimental results on three benchmarks --- racing car, cart pole, and donkey car --- show that our UDA-equipped evaluators maintain high accuracy and substantially lower false positive rates under distribution shift. Similarly, world model-based composite predictors outperform their monolithic counterparts on long-horizon tasks, and our conformal calibration provides reliable statistical bounds.}


\maketitle

\section{Introduction}\label{sec:introduction}

Autonomous robots increasingly rely on high-resolution sensors (e.g., cameras, lidars) and deep-learning architectures~\cite[]{grigorescu_survey_2020,yurtsever_autonomous_2020,liu_multimodal_2024}. End-to-end reinforcement and imitation learning have become widespread, bypassing traditional components such as state estimation and planning~\cite[]{bojarski_end_2016,codevilla_conditional_2018,kiran_rl_survey_2021}.
Although effective, these end-to-end systems often depend on complex, opaque models whose internal decision-making processes are difficult to interpret~\cite[]{kuznietsov_explainable_2024}. This lack of transparency, combined with the inherent variability and unpredictability of real-world conditions, makes it difficult to anticipate potential failures before they manifest~\cite[]{hecker_failure_2018,filos_distribution_shift_2020}. Failures in autonomous driving can lead to catastrophic accidents, which erode public trust and hinder broader adoption~\cite[]{penmetsa_effects_2021}. Robust safety prediction is therefore essential to ensure the reliability and acceptance of these systems, making it both a technical hurdle and a societal imperative.

In the autonomous driving industry, vision-based frameworks such as Tesla Vision demonstrate the feasibility of end-to-end control using only camera inputs ~\cite[]{tesla_vision_2025}. By learning scene geometry and motion directly from raw images, such systems remove map and sensor dependencies and adapt more readily to new roads and dynamic conditions. However, this image-only paradigm amplifies the need for robust, well-calibrated safety prediction, as perception and control are fully integrated and lack explicit, low-dimensional state representations.

Moreover, in safety-critical settings, quantifying uncertainty becomes essential to operate under partial observability or environmental variability. However, modern deep models often produce overconfident or miscalibrated outputs, undermining reliability~\cite[]{guo_calibration_2017,minderer_revisiting_2021}. In autonomous robots, this can result in unsafe behavior or missed interventions~\cite[]{lutjens_safe_2018,feng_can_2019}. Thus, to ensure safety, predictions must be accompanied by confidence estimates that accurately reflect real-world outcomes.

\begin{figure}[t]
    \centering
    \includegraphics[width=0.8\columnwidth]{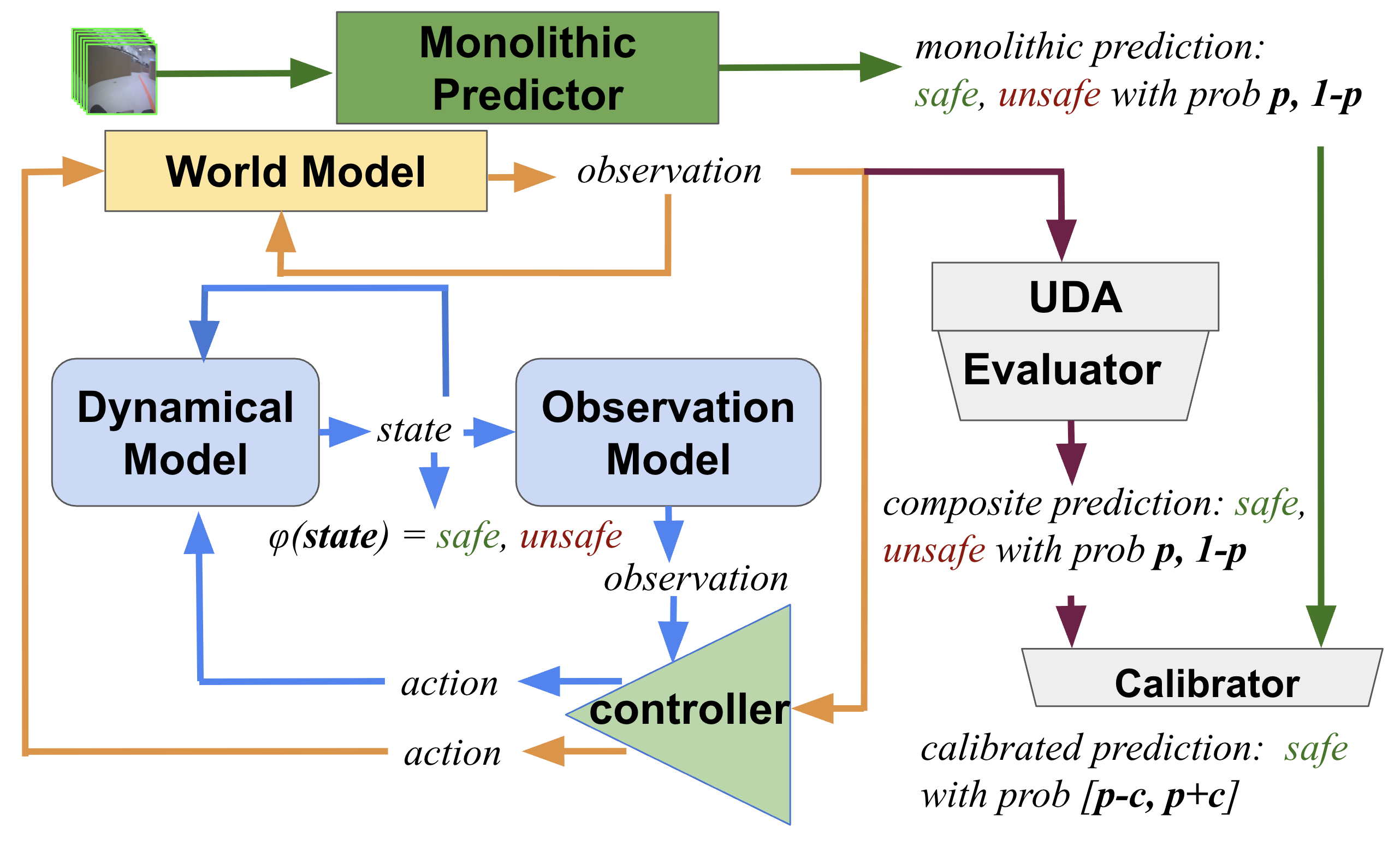}
    \caption{Safety prediction pipelines for image-controlled autonomous systems.
    The \textcolor{blue}{blue arrows} represent the underlying dynamics loop, consisting of the unknown dynamical model, unknown observation model, and known controller.
    The \textcolor{orange}{orange arrows} indicate the learned \textit{world model loop}, where future observations are simulated based on past observations and actions.
    The \textcolor{ForestGreen}{green path} illustrates the \textit{monolithic prediction pipeline}, which directly predicts safety outcomes (safe/unsafe with probability \(p\)) from observations.
    The \textcolor{purple}{purple path} shows the \textit{composite prediction pipeline}, where a forecaster predicts future observations that are passed to a safety evaluator.}
    \label{fig:overview}
\end{figure}

Many existing safety prediction methods depend on low-dimensional, physically meaningful states (e.g., poses and velocities) and hand-crafted indicators that require extensive domain knowledge and offline analysis, making them impractical for online application to high-dimensional, image-based controllers ~\cite[]{ramesh_robot_vitals_2022,gautam_know_limits_2022,ruchkin_confidence_2022,conlon_event_triggered_self_assessment_2024}. Additionally, model-based offline verification scales poorly to vision systems, as it demands detailed modeling of the perception subsystem~\cite[]{everett_efficient_2021,toufighi_decision_2024,mitra_formal_2025}. Similarly, black‐box statistical risk assessments still tend to assume low‐dimensional state representations and rich outcome labels (e.g., signal robustness)~\cite[]{donze_robust_2010,muller_situational_risk_2022,lindemann_conformal_2023}.
Recent graph‐based models HDGT, FutureNet, and SemanticFormer represent the road environment as a graph whose nodes correspond to map elements (e.g., lanes, intersections) and dynamic agents, and whose edges encode spatial or temporal relationships~\cite[]{li_hdgt_2022,chen_futurenet_2023,wang_semanticformer_2024}. These models use the graph structure to perform multi‐agent trajectory prediction conditioned on HD maps, precomputed lane graphs, and multiple sensors inputs. This dependence incurs high annotation and maintenance costs and limits application to unmapped or rapidly changing environments.

However, purely data-driven, image-based predictors themselves can suffer under distribution shift, leading to miscalibrated risk estimates and degraded accuracy when encountering novel environmental conditions~\cite[]{filos_distribution_shift_2020,yang2024generalized}. Under these circumstances, the predicted vision trajectory cannot be accurately classified by a visual safety evaluator, where the small errors compound over time, producing visual trajectories whose distribution deviates substantially from the training data. As a result, a static evaluator trained offline often fails to maintain calibration and exhibits inflated false-positive rates, which directly undermines reliability in safety-critical decisions. This motivates the need for an adaptive mechanism that can align the evaluator to the evolving input distribution without manual relabeling~\cite{mao2024safe}.

This article addresses the problem of \textbf{reliable online safety prediction} without access to a known, low-dimensional dynamical state. Instead, our predictor relies on a rolling window of images, for example, the RGB frames from the most recent mm time steps of a race car’s run.
Given this image history, our goal is to predict the probability that the car will stay on track for the next \(k\) frames. We also require that this predicted probability is \textit{calibrated well}; that is, the events forecast with probability \(p\) should occur approximately a \(p\) fraction of the time.

\looseness=-1
We introduce a flexible family of learning pipelines, shown in Figure~\ref{fig:overview}, for image-based safety prediction, featuring:
(i) \emph{deep world models} that learn compact latent representations from raw observations and jointly model the dynamics to predict future observations;
(ii) monolithic models that \emph{capture temporal evolution} of latent safety-critical features without reconstructing full observations;
(iii) an \emph{adaptive UDA module} that detects distribution shifts and selectively fine-tunes the safety evaluator online to maintain calibration on shifted inputs and
(iv) \emph{post-hoc conformal calibration method} that converts raw scores into trustworthy probabilities independently of the underlying architecture.

We evaluate our pipeline on three benchmarks of long-horizon safety prediction: racing car and cart pole from simulated OpenAI Gym environments and a physical deployment of camera-driven racing cars, known as Donkey Cars~\cite[]{viitala_learning_2021,viitala_scale_2021}.
Our results reveal the advantages of modular pipelines with attention mechanisms and show that predicting calibrated probabilities is easier than predicting binary safety labels over extended horizons.

In summary, this article makes three contributions:
\begin{enumerate}
    \item A family of \emph{modular pipelines for safety prediction} in image-controlled autonomous robots.
    \item A \emph{conformal post-hoc calibration} technique to equip safety estimates with statistical guarantees.
    \item A \emph{comprehensive, realistic experimental evaluation} that highlights the impact of modularity and calibration on long-horizon safety prediction.
\end{enumerate}

An earlier iteration of this work appeared in a conference paper~\cite{mao2024safe}. This journal article substantially extends the conference version in three key ways:
\begin{enumerate}
    \item We develop a UDA method for making safety evaluators robust to distribution shift. In addition to our original image-based evaluator, we introduce a novel \textit{latent-space safety evaluator}, enabling more robust post-hoc safety estimation for forecasted visual observations and latent states.
    \item We improve the long-horizon performance of our pipelines with the latest world-model architectures: bidirectional LSTMs, Vector Quantized VAEs (VQ-VAEs), and Transformers.
    \item We introduce a new physical case study of vision-based Donkey Cars to evaluate our predictors on a realistic high-speed autonomous racing.
\end{enumerate}

\section{Related Work}\label{sec:relwork}

\subsection{Trajectory Prediction}\label{subsec:trajpred}

A common way of predicting the system's performance and safety is by inferring it from \emph{predicted trajectories}.
Classic approaches consider model-based prediction, for instance, estimating collision risk with bicycle dynamics and Kalman filtering~\cite[]{5284727,lefevre_survey_2014}. One common approach with safety guarantees is \emph{Hamilton-Jacobi} (HJ) reachability, which requires precomputation based on a dynamical model~\cite[]{li_prediction-based_2021,nakamura2023online}.
Among many deep learning-based predictors, a popular architecture is \emph{Trajectron++}, which takes in high-dimensional scene graphs and outputs future trajectories for multiple agents~\cite[]{huang_survey_2022,salzmann_trajectron_2020}. A conditional VAE is adopted in Trajectron++ to add constraints at the decoding stage~\cite[]{NIPS2015_8d55a249}. This can predict more than one reasonable trajectory option under the encoding condition, which helps analyze the future safety under different ego-vehicle intentions. Learning-based trajectory predictions can be augmented with conformal prediction to improve their reliability~\cite[]{lindemann_conformal_2023,muthali_multi-agent_2023}. These techniques have also been adapted to motion planning settings, where adaptive conformal prediction provides calibrated guarantees for multi-agent collision avoidance~\cite[]{dixit2023adaptive}. Recent work has further emphasized generalizability and interpretability as key desiderata for trajectory predictors, advocating for Bayesian latent-variable models to provide better uncertainty modeling and extrapolation across scenarios~\cite[]{lu2024towards}.

In comparison, our work eschews handcrafted scene and state representations, instead using image and latent representations that are informed by safety. This trend is echoed in efforts to design self-aware and interpretable neural networks~\cite[]{itkina2023interpretable} for robust trajectory prediction under uncertainty and interpretable world models using physically meaningful representations under dynamical constraints~\cite[]{peper2025four,mao2024towards}.

\subsection{Performance and Safety Evaluation}\label{subsec:perfsafe}

A variety of recent research enables robotic systems to \emph{self-evaluate} their competency~\cite[]{basich_competence-aware_2022}. Performance metrics vary significantly and include such examples as the time to navigate to a goal location or whether a safety constraint was violated, which is the focus of our paper. For systems that have access to their low-dimensional state, such as the velocity and the distance from the goal, \emph{hand-crafted indicators} have been successful in measuring performance degradation. These indicators have been referred to robot vitals, alignment checkers, assumption monitors, and operator trust~\cite[]{9834068,9812030,ruchkin_confidence_2022,conlon_im_2022}. Typically, these indicators rely on domain knowledge and careful offline analysis of the robotic system, which are difficult to respectively obtain and perform for high-dimensional sensors.

Deep neural network (DNN) controllers are particularly vulnerable to distribution shift and difficult to formally analyze~\cite[]{MORENOTORRES2012521,HUANG2020100270}. This challenge, combined with the critical need for safety and performance guarantees, has driven the development of both model-based and model-free prediction approaches. On the model-based side, \emph{closed-loop verification approaches} perform reachability analysis to evaluate the safety of a DNN-controlled system~\cite[]{ivanov_verisig_2021,tran_verification_2022}; however, when applied to vision-based systems, these techniques require detailed modeling of the vision subsystem and have limited scalability~\cite[]{santa_cruz_nnlander-verif_2022,hsieh_verifying_2022,mitra_formal_2025}. Furthermore, these closed-loop verification methods fundamentally require an explicit low-dimensional dynamical model to perform reachability analysis; our framework operates end-to-end from image sequences to safety predictions without access to such a model, making these methods architecturally incompatible with our setting~\cite[]{ivanov_verisig_2021,tran_verification_2022}. On the other hand, \emph{model-free safety predictions} rely on the correlation between performance and uncertainty measures/anomaly scores, such as autoencoder reconstruction errors and distances to representative training data~\cite[]{yang_interpretable_2022}.
Striking the balance between model-based and model-free approaches are \emph{black-box statistical methods} for risk assessment and safety verification~\cite[]{cleaveland_risk_2022,10.1145/3365365.3382209,9797578,michelmore_uncertainty_2020}, such as uncertainty quantification with statistical guarantees in end-to-end vision-based control~\cite[]{michelmore_uncertainty_2020}.
These methods require low-dimensional states and rich outcome labels (such as signal robustness) --- an assumption that we are relaxing in this work while addressing a similar problem.

\noindent
In addition, safe planning and control solve a problem complementary to ours: they find actions to control a robot safely, usually with respect to a dynamical model, such as motion planning under uncertainty for uncertain systems and a growing body of work on control barrier functions~\cite[]{knuth_planning_2021,hibbard_safely_2022,chou_synthesizing_2023,ames_control_2019,xiao_safe_2023}. Some approaches are robust to measurement errors induced by learning-based perception, but at their core, they still hinge on the utilization of low-dimensional states, a paradigm we aim to circumvent entirely~\cite[]{dean_guaranteeing_2021,yang_safe_2023}. Importantly, one of our constraints is assuming the existence of an end-to-end controller (possibly implemented with the above methods), with no room for modifications.

In the context of confidence calibration, the softmax scores of classification neural networks can be interpreted as probabilities for each class; however, standard training leads to miscalibrated neural networks, as measured by \emph{Brier score} and \emph{Expected Calibration Error} (ECE)~\cite[]{guo_calibration_2017,minderer_revisiting_2021}. Calibration approaches can be categorized into \emph{extrinsic (post-hoc) calibration} added on top of a trained network, such as Platt and temperature scaling, isotonic regression, and histogram/Bayesian binning---and \emph{intrinsic calibration} to modify the training, such as ensembles, adversarial training, and learning from hints, error distances, or true class probabilities~\cite[]{platt_probabilistic_1999,guo_calibration_2017,zadrozny_transforming_2002,naeini_obtaining_2015,zhang_mix-n-match_2020,lee_training_2018,devries_learning_2018,xing_distance-based_2019,corbiere_addressing_2019}. Similar techniques have been developed for calibrated regression~\cite[]{vovk_conformal_2020,marx_modular_2022}. It is feasible to obtain calibration guarantees for safety chance predictions, but it requires a low-dimensional model-based setting~\cite[]{ruchkin_confidence_2022,cleaveland_conservative_2023}. While recent work focuses on conformal calibration of temporal logic satisfaction in structured, model-based systems, our work addresses a different challenge: calibrating safety chance predictions directly from image sequences in end-to-end controlled systems without access to low-dimensional state or explicit logical specifications~\cite{lindemann_conformal_2023}. To the best of the authors' knowledge, such guarantees have not been instantiated in a model-free autonomy setting.
\subsection{World Models}\label{subsec:worldmodels}

For systems with high-resolution images and complex dynamical models, trajectory prediction becomes particularly challenging. Physics-based methods and classical machine learning methods may not be applicable to these tasks, leading to the almost exclusive application of deep learning for prediction. \emph{Sequence prediction models}, which originate in deep video prediction, often struggle in our setting due to compounding prediction errors, the difficulty of modeling long-term dependencies, and the large data requirements for learning high-dimensional dynamics~\cite[]{Oprea_2022}. However, incorporating additional information, such as states, observations, and actions, in sequence predictions can improve their performance --- and our predictors take advantage of that~\cite[]{oh2015action,finn2017deep,strickland2018deep}. Another way to improve performance is to reduce the input dimensionality. High-dimensional sensor data often contains substantial redundant and irrelevant information, which leads to higher computational costs. Generative adversarial networks (GANs) map the observations into a low-dimensional latent space, which can enable model-based assurance algorithms, and low-dimensional representations can incorporate conformal predictions to monitor the real-time performance~\cite[]{goodfellow2020generative,katz2021verification,boursinos_assurance_2021}.
\emph{World models}~\cite[]{ha_recurrent_2018} are a model-based approach in which a learned environment model serves as a surrogate for the real environment, allowing the controller to be trained largely in imagination rather than through direct interaction with the physical world. This can substantially improve sample efficiency and stability compared to training solely in the real environment. They achieve better results in training the controllers than basic RL. The world model learns an approximation of the system’s dynamics by collecting observations and actions as inputs and outputs the future states. The observation encoder, typically implemented as a VAE, compresses each observation into a latent space (low-dimensional vectors) and trains a mixed-density recurrent neural network (MDN-RNN) to forecast the next-step latent observations. The decoder of the VAE maps the latent vector into the observation space. This VAE--MDN-RNN world model is then applied to do trajectory predictions and get the competency assessment through the forecasting trajectory~\cite[]{acharya_competency_2022}.

Recent advances explore the use of discrete latent spaces for more structured and sample-efficient planning, including transformer-based discrete world models for continuous control, building on neural discrete representation learning and echoing our use of VQVAE-based encoders~\cite{scannell2025discrete,van2017neural}. Building on the original VAE--MDN-RNN architecture of world models, more advanced models like \emph{DreamerV2} and \emph{IRIS} adopt recent deep computer vision models like vision transformers to achieve more vivid images~\cite[]{ha_recurrent_2018,hafner2022mastering,micheli2023transformers}. Transformers have also been shown to improve sample efficiency in latent model training, a crucial factor for model-based planning and control~\cite[]{micheli2022transformers}. Recent systems such as DayDreamer extend this line of work by leveraging compact world models for sample-efficient control in physical robots, demonstrating the practicality of Dreamer-style pipelines~\cite[]{wu2023daydreamer}. More recently, Orbis compares different latent spaces within a unified hybrid tokenizer, showing that a flow-matching-based continuous model achieves superior long-horizon prediction in challenging driving scenarios using only raw video data~\cite{mousakhan2025orbis}.

The trajectories predicted by world models have been used to provide uncertainty measurements for agents, albeit without statistical guarantees so far~\cite[]{acharya2023learning}. In our method, we adopt different architectures of world models and compare them with other one-step monolithic deep learning models using the recurrent modules. Building on this foundation, recent work has demonstrated that pre-trained, interpretable world models can enable zero-shot safety prediction from visual inputs, and that weakly supervised latent representations can further improve physical consistency and trajectory-level accuracy~\cite[]{mao2024zero,mao2024towards}. These efforts are aligned with emerging principles for physically interpretable world models, which call for organized latent spaces, aligned invariant/equivariant representations, integration of diverse supervision, and partitioned outputs for scalability~\cite[]{peper2025four}. Compared with these approaches, this article develops chance calibration guarantees, which were previously unexplored.

\subsection{Unsupervised Domain Adaptation}\label{subsec:uda}

Predicting future observations from high-dimensional sensory inputs, such as image sequences, inevitably leads to distribution shift, especially over long prediction horizons. This shift may arise from subtle environmental variations, motion blur, sensor noise, or compounding prediction errors. When applied to such shifted predictions, safety evaluators often produce unreliable or overconfident outputs, which compromise downstream decision-making~\cite[]{yang2024generalized,stocco2020misbehaviour}.
To address this, we adopt a principled solution of \textit{unsupervised domain adaptation (UDA)} that continuously adapts the safety evaluator to the evolving input distribution without requiring manual relabeling. This enables the system to maintain calibration guarantees even as the environment changes or predictions drift from training conditions.

A range of UDA strategies has been developed across machine learning and robotics. Early methods, such as Deep CORAL and MMD-based alignment, align source and target feature distributions via statistical matching~\cite[]{sun2016deep,long2015learning}. Adversarial approaches like DANN and ADDA train domain-invariant representations by using domain discriminators~\cite[]{ganin2016domain,tzeng2017adversarial}. Others operate at the image level using GAN-based transformations~\cite[]{bousmalis2017unsupervised,zhu2017unpaired,sobolewski2025generalizable}. More recent work enables adaptation at test time, including entropy-based methods like Tent and contrastive self-supervised approaches such as CoSCA~\cite[]{wang2020tent,kang2019contrastive}. In our work, we adopt a lightweight method using the entropy minimization (MEMO) approach, which does not require source data during adaptation~\cite[]{zhang2022memo}. It is well-suited for safety evaluation in long-horizon prediction settings because it can be applied efficiently within the prediction pipeline without querying external data.

\section{Background and Problem}\label{sec:prelim}

This section introduces the necessary notation and describes the problem addressed in this article.

\subsection{Problem Setting}\label{subsec:problemsetting}

\begin{definition}[Robotic system]
A discrete-time \emph{robotic system} $s=(\statespace,\obsspace,\actspace,\imgcontroller,\dynmodel,\obsmodel,\initstate,\safeprop)$ consists of:
\begin{itemize}
    \item \emph{State space} \statespace, containing continuous states $x$
    \item \emph{Observation space} \obsspace, containing images $y$
    \item \emph{Action space} \actspace, with commands $u$
    \item \emph{Image-based controller} $\imgcontroller:\mathbf \obsspace \rightarrow \mathbf \actspace$, typically implemented by a neural network
    \item \emph{Dynamical model} $\dynmodel: \statespace \times \actspace \rightarrow \statespace$, determines the next state (unknown to us)
    \item \emph{Observation model} $\obsmodel:\statespace \rightarrow \obsspace$, which generates an observation based on the state  (unknown to us)
    \item \emph{Initial state} $ \initstate$, from which the system starts executing
    \item \emph{State-based safety property} $\safeprop: \mathbf{X} \rightarrow \{0,1\}$, which determines whether a given state $x$ is safe
\end{itemize}
\end{definition}

We focus on robotic systems with observation spaces with thousands of pixels and unknown non-linear dynamical and observation models. In such systems, while the state space \statespace is conceptually known (if only to define \safeprop), it is not necessary (and often difficult) to construct \dynmodel and \obsmodel because the controller acts directly on the observation space \obsspace. Without relying on \dynmodel and \obsmodel, end-to-end methods like deep reinforcement learning and imitation learning are used to train a controller \imgcontroller by using data from the observation space ~\cite[]{mnih2015human,hussein_imitation_2017}. Once the controller is deployed in state \initstate, the system executes a \emph{trajectory}, which is a sequence $\{x_i,y_i,u_i\}^{t}_{i=0}$ up to time $t$, where:
\begin{align}
x_{i+1} = \dynmodel(x_i, u_i), \quad  y_i = \obsmodel(x_i), \quad
u_i = \imgcontroller(y_i)
\end{align}

We consider robotic systems where the underlying dynamical model \dynmodel and observation model \obsmodel are unknown.
This means we cannot explicitly use the known dynamical and observation models to predict future safety. Instead, we extract predictive information from three sources of data. First, we will use the current observation, $y_i$, at some time $i$.
For example, when a car is at the edge of the track, it has a higher probability of being unsafe in the next few steps.
Second, past observations $y_{i-m+1},\dots, y_i$ provide dynamically useful information that can only be extracted from time series, such as the speed and direction of motion. For example, just before a car enters a turn, the sequence of past observations implicitly informs how hard it will need to be to remain safe.
Third, not only do observations inform safety, but so do the controller's past/present outputs $\imgcontroller(y_{i-m+1},\dots, y_i)$. For instance, if by mid-turn the controller has not changed the steering angle, it may be less likely to navigate this turn safely. In addition, forecasting future control actions $(\hat{u}_{i+1}, \dots, \hat{u}_{i+k})$ would provide another input to predict the trajectory.

Our goal is to predict the system's safety $\safeprop (x_{i+k})$ at time $i+k$
given a series of $m$ observations $\mathbf y_i=(y_{i-  m+1},...,y_i)$.
This sequence of observations does not, generally, determine the true state $x_i$ (e.g., when $m=1$, the function \obsmodel may not be invertible). This leads to \emph{partial state observability}, which we model stochastically. Specifically, we say that from the predictor's perspective, $x_i$ is drawn from some belief distribution $\mathcal{D}_{\mathbf y_i}$. This induces a distribution of subsequent trajectories and transforms future safety $\safeprop (x_{i+k})$ into a Bernoulli random variable. Therefore, we will estimate the conditional probability $P(\safeprop (x_{i+k}) \mid \mathbf y_i)$ and provide an error bound on our estimates. Note that process noise in $d$ and measurement noise in $o$ are both supported by our approach and would be treated as part of the stochastic uncertainty in the future $\safeprop$. To sum up the above, we arrive at the following problem description.

\begin{problem}[Calibrated safety prediction]
Given horizon $k>0$, confidence $\alpha \in (0, 0.5)$, and observations $\mathbf y_i$ from some system $s$ with unknown \dynmodel and \obsmodel, estimate future safety chance $P(\safeprop(x_{i+k})\mid \mathbf y_i)$ and provide an upper bound for the estimation error that holds for at least $1-\alpha$\% of estimates.
\end{problem}

In addition, instead of only outputting an estimated safety chance, we aim to provide a guaranteed interval around it.

\begin{problem}[Probability interval prediction]
Given the same inputs as Problem 1 and an estimate
$P'(\safeprop(x_{i+k}) \mid \mathbf{y}_i)$ from a safety chance predictor, determine an interval over the estimated confidence with width $c$ such that it contains the true chance $P(\safeprop(x_{i+k}) \mid \mathbf{y}_i)$ with confidence $\alpha$:
\begin{align*}
  \mathbb{P}\Big( | P'(\safeprop(x_{i+k}) \mid \mathbf{y}_i) - P(\safeprop(x_{i+k}) \mid \mathbf{y}_i) | \le c  \Big) \ge 1 - \alpha.
\end{align*}
\end{problem}

\subsection{Predictors and Datasets}\label{subsec:predictorsdatasets}

\looseness=-1
To address our problem, we consider two types of safety predictors: one that outputs a binary safety label, and another that outputs the probability of being safe in the future.

\begin{definition}[Safety label predictor]\label{def:predictor-label}
For horizon $k > 0$, a \emph{safety label predictor} $\labelpred: \obsseqspace \rightarrow \{0, 1\}$ predicts the binary safety outcome $\safeprop(x_{i+k})$ at time $i+k$.
\end{definition}

\begin{definition}[Safety chance predictor]\label{def:predictor-chance}
For horizon $k > 0$, a \emph{safety chance predictor} $\chancepred: \obsseqspace \rightarrow [0, 1]$ estimates the probability $P'(\safeprop(x_{i+k}) \mid \mathbf{y}_i)$ that the system remains safe at time $i+k$.
\end{definition}

\begin{definition}[Safety chance interval predictor]\label{def:predictor-interval}
For horizon $k > 0$, a \emph{safety interval predictor} $\eta: \obsseqspace \rightarrow [0, 1] \times [0, 1]$ estimates the interval of probabilities $[P'(\safeprop(x_{i+k}) \mid \mathbf{y}_i) - c, P'(\safeprop(x_{i+k}) \mid \mathbf{y}_i) +c]$ that the system remains safe at time $i+k$ such that the true probability $P(\safeprop(x_{i+k}) \mid \mathbf{y}_i)$ lies within that interval with confidence at least $1 - \alpha$, where $\alpha \in (0, 0.5)$.
\end{definition}

Both types of predictors are trained using labeled \textit{observation-action datasets}, which combine each observation with the controller’s executed action and the eventual safety label. This setup enables training without sacrificing predictive fidelity, as actions encode critical behavioral information even in visually ambiguous settings. Thus, all models described in this article are trained on datasets of the following form:

\begin{definition}[Observation-action dataset]\label{def:oad}
An \emph{observation-action dataset} $\dataset = \{ (\dataseq_j, \safeprop_j) ~|~ j = 1, \dots, N \}$ consists of $m$-length sequences of image-action pairs
\[
\dataseq_j = ((y_{i-m+1}, u_{i-m+1}), \dots, (y_i, u_i))
\]
and a safety label $\safeprop_j := \safeprop(x_{i+k})$ obtained $k$ steps later.
\end{definition}

Notice that our dataset is structured not around sequential states within a single trajectory, but as independent trajectories sampled from the same distribution. This ensures that the calibration and test samples are i.i.d., satisfying the exchangeability requirement of conformal prediction. We are not applying conformal prediction to temporally correlated states along a single trajectory; rather, each data point represents an independent realization of the same prediction task.

\section{Modular Family of Safety Predictors}\label{sec:approach}

We begin by describing the construction of \emph{safety-label predictors} and then show how they are upgraded to \emph{safety-chance predictors} that quantify uncertainty.

\subsection{Monolithic and Composite  Predictors}\label{subsec:mono-comp}

\begin{figure*}[t]
    \centering
    \includegraphics[width=\columnwidth]{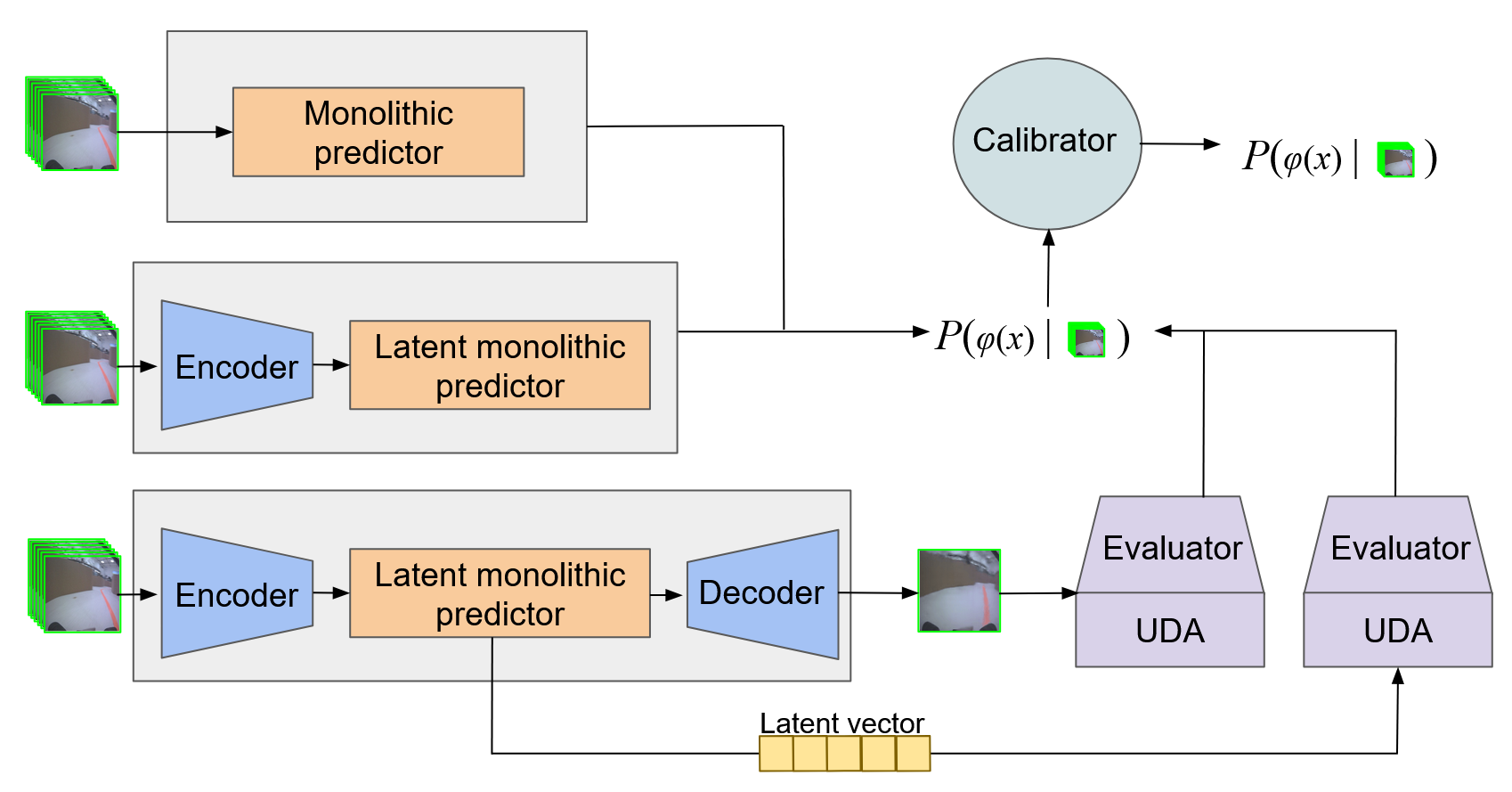}
    \caption{\looseness=-1
    Comparison of safety prediction pipelines and their world-model structures (gray boxes).
    \textbf{Top:} The \emph{monolithic predictor} maps image sequences directly to safety probabilities \(P(\varphi(x))\), followed by a \textit{calibrator} for improved confidence.
    \textbf{Middle:} The \emph{latent monolithic predictor} encodes images into latent states and predicts safety from these, with calibration applied afterward.
    \textbf{Bottom:} The \emph{composite latent predictor} encodes inputs, forecasts future latent states, and uses an \textbf{evaluator} with UDA to estimate safety.
    All pipelines predict the probability of satisfying the safety property \(\varphi(x)\).
    }
    \label{fig:comp}
\end{figure*}

Our family of safety predictors rests on a simple modularity principle: any safety-label predictor
can be decomposed into three conceptually independent blocks that communicate
through a latent space~$\Theta$. An \emph{encoder}~$E$ maps each observation
\(y_i \in \mathcal{Y} \) to a latent state \( \theta_i \in \Theta \).
A \emph{forcaster}~\(P\) is a generic sequence operator that consumes a window
of past states and returns a forecast \( k \) steps ahead; its internal
mechanism is not fixed: it can be recurrent, convolutional, attentional, or
something else entirely. Finally, an \emph{evaluator}~\(V\) assigns a binary
safety label to the forecast. Similar to the predictor, its internal mechanism is not fixed.

Since \(P\) is abstract, the same modular decomposition and reasoning apply whether it is implemented
as an LSTM, a GRU, a temporal convolution, a Transformer encoder stack, or any
future architecture with the requisite input–output signature. In what
follows, we treat \(P\) entirely as a black box. The only place we commit to a
concrete realisation is in the experimental section, where we provide
numerical results for two instantiations --- an LSTM and a
Transformer --- chosen to illustrate, respectively, recurrence and attention.

Two architectural designs arise from our modular view. A \emph{monolithic} predictor applies the encoder \(E\) and immediately feeds its latent output to the evaluator \(V\), so the safety label is produced in a single forward pass with no separate forecasting block. A \emph{composite} predictor inserts an explicit forecasting stage: it first reconstructs a
sequence of latent states and then
feeds those predictions to \(V\) for evaluation. Figures~\ref{fig:overview} and~\ref{fig:comp} illustrate both dataflows.

\begin{definition}[Monolithic latent predictor]\label{def:monolithic-predictor}
Let $E:\mathcal{Y} \to \Theta$ be an encoder and
$V:\Theta \to \{0,1\}$ an evaluator.
The composition $g = V \circ E$ is called a
\emph{monolithic latent predictor}.%
\end{definition}

\looseness=-1
Composite pipelines explicitly separate forecasting from evaluation. In the image space, a forecaster \(P_g\) predicts future frames that are then classified by a CNN. In the latent space, a forecaster \(P_l\) outputs \((\theta'_{i+1}, \ldots, \theta'_{i+k})\), after which \(V\) inspects the latent trajectory. Similarly, the design of the composite pipeline does not restrict how the forecaster is built; it may rely on recurrence, attention, or any hybrid thereof.


\begin{definition}[Composite image predictor]\label{def:composite-image}
A composite image predictor comprises two learned modules:
(i) an \emph{image forecaster} \(P_g : \mathcal{Z}_m \to \mathcal{Y}\) that
maps the past \(m\) observations \(z_i\) to a forecasted image
\(\hat y_{i+k}\), and
(ii) an \emph{evaluator} \(V : \mathcal{Y} \to \{0,1\}\) that assigns a
binary safety label to each predicted image.
The overall prediction is
\[
g(z_i) \;=\; V\!\bigl(F_g(z_i)\bigr).
\]
\end{definition}

\begin{definition}[Composite latent predictor]\label{def:composite-latent}
A composite latent predictor consists of three learned modules:
(i) an \emph{encoder} \(E : \mathcal{Z} \to \Theta\) that maps each
observation \(z_i\) to a latent state \(\theta_i \in \Theta\);
(ii) a \emph{latent forecaster}
\[
P_l :
\underbrace{\Theta \times \cdots \times \Theta}_{m\ \text{past states}}
\;\longrightarrow\;
\underbrace{\Theta \times \cdots \times \Theta}_{n\ \text{predicted states}}
\]
that, given \((\theta_{i-m+1},\dots,\theta_i)\), produces the sequence
\((\hat\theta_{i+k-n+1},\dots,\hat\theta_{i+k})\); and
(iii) an \emph{evaluator} \(V : \Theta \to \{0,1\}\) that labels each predicted
state. Thus, the composite prediction is denoted as
\[
V\!\bigl(\hat\theta_{i+k-n+1},\dots,\hat\theta_{i+k}).
\]
\end{definition}

\looseness=-1
For the latent forecaster~$P_l$, we compare three architectures.
The first is a unidirectional LSTM, which provides a strong recurrent baseline. The second extends this with a bidirectional LSTM that incorporates a backward pass, allowing predictions during training to condition on both preceding and succeeding elements of the input sequence~\cite[]{schuster_bidirectional_1997}. This enables the model to construct richer temporal representations with enhanced coherence and accuracy of the forecasted output window. Lastly, we adopt a Transformer-based architecture in which each layer combines multi-head self-attention with a position-wise feed-forward network, connected via residual links and layer normalization~\cite[]{vaswani_attention_2017}. This design allows every time step in the input window to attend to all others within the same window, enabling the model to capture complex long-range dependencies more effectively than recurrent models.
Once the forecast $(\hat\theta_{i+k-n+1},\dots,\hat\theta_{i+k})$ is produced, the evaluator~$V$ emits only a single binary label and does not benefit from future context; therefore,
we implement $V$ as a lightweight unidirectional LSTM, which is both
parameter-efficient and sufficient for this one-step task.

In the composite latent predictor pipeline, in addition to the standard continuous-latent VAE encoder, we also evaluate a discrete-latent variant based on the Vector Quantized VAE (VQ-VAE) architecture~\cite[]{van2017neural}. The VQ-VAE maps each observation into a sequence of discrete codes drawn from a learned codebook, which serves as a compact, structured representation of the input. This comparison allows us to assess whether discrete latent representations offer improved stability and safety-prediction accuracy in the composite pipeline.

\subsection{Training Process}\label{subsec:training}

Monolithic predictors $g$, evaluators $V$, and image forecasters $P_{g}$ are trained in a supervised manner on a training dataset $Z_{tr}$ containing observation sequences with associated safety labels or future images, respectively. For the composite models, we adopt a two‐stage training procedure. First, the VAE encoder $E$ is trained (as described in the next paragraph) to produce latent representations of observations. Then, the latent forecasters $P_{l}$ are trained on future latent vectors, obtained from true future observations with a VAE encoder $E$. The \textit{mean squared error (MSE)} loss is implemented for the LSTM and Transformer latent forecasters, since the MSE loss directly penalizes deviations between predicted and true continuous latent vectors. The monolithic predictors and evaluators use \textit{cross‐entropy (CE)} as the loss function, whereas it is suited for classification tasks like safety evaluation, as it maximizes the probability of the correct class. We also implement an early stopping strategy to reduce the learning rate as the loss drop slows down --- if the average loss improvement over the past 10 consecutive epochs is smaller than $10^{-4}$, we stop the training. Additionally, the learning rate is reduced when the loss drop slows down to encourage finer convergence.

For the VAE, the total loss $L$ consists of two parts. The first is the reconstruction loss $L_{\mathrm{recon}}$, which quantifies how well an original image $y$ is approximated by $d(e(y))$ using the MSE loss. The second is the latent loss $L_{\mathrm{latent}}$, which uses the Kullback–Leibler divergence loss to minimize the difference between two latent‐vector probability distributions $e(y \mid \theta)\approx d(\theta \mid y)$, where $\theta$ is a latent vector and $y$ is an input. Thus, the total loss function for the VAE combines the three parts with a regularization parameter $\lambda_{1} > 0$:
\begin{equation}\label{eq:2}
L \;=\; L_{\mathrm{recon}} \;+\; \lambda_{1}\,L_{\mathrm{latent}} \;
\end{equation}

One challenge is that the safety label balance changes in $Z_{tr}$ with the prediction horizon $k$, leading to imbalanced data for higher horizons. To ensure a balanced label distribution, we resample with replacement for a 1:1 safe:unsafe class balance in the training data.

\begin{algorithm}[t]
\caption{Conformal calibration for chance predictions}
\label{alg:concali}
\begin{algorithmic}[1]
    \Require A validation dataset bin \(B = \{b_{k}\}_{k=1,\dots}\), each containing sequences of observations and safety \(b_{k} = (y_{k}, \phi(x_{k}))\), trained safety chance predictor \(g\), and miscoverage level \(\alpha\)
    \Ensure confidence bound \(c\) satisfying Equation.~4.
    \State \textbf{Function} ConCali$(B, g, \alpha)$:
    \For{$j = 1$ to $Q$}
        \For{$i = 1$ to $M$}
            \State $B_{i} \leftarrow N$ i.i.d.\ samples from $B$ \{resampled bin\}
            \State $q_{i} \leftarrow \frac{1}{N}\sum_{l=1}^{N} g(y_{l}),\;\text{for each }y_{l}\in B_{i}$ \{mean safety chance prediction\}
            \State $p_{i} \leftarrow \frac{1}{N}\sum_{l=1}^{N} \phi(x_{l}),\;\text{for each }\phi(x_{l})\in B_{i}$ \{true safety chance\}
            \State $\delta_{i} \leftarrow \lvert q_{i} - p_{i}\rvert$ \{non‐conformity score\}
        \EndFor
        \State $n \leftarrow \lceil (M + 1)(1 - \alpha)\rceil$ \{conformal quantile\}
        \State $c \leftarrow$ the $n$‐th smallest value among $\{\delta_{1}, \ldots, \delta_{M}\}$
    \EndFor
    \State \textbf{return} $c$
\end{algorithmic}
\end{algorithm}
\subsection{Unsupervised Domain Adaptation}\label{subsec:uda}

\begin{figure}[t]
    \centering
    \includegraphics[width=0.7\columnwidth]{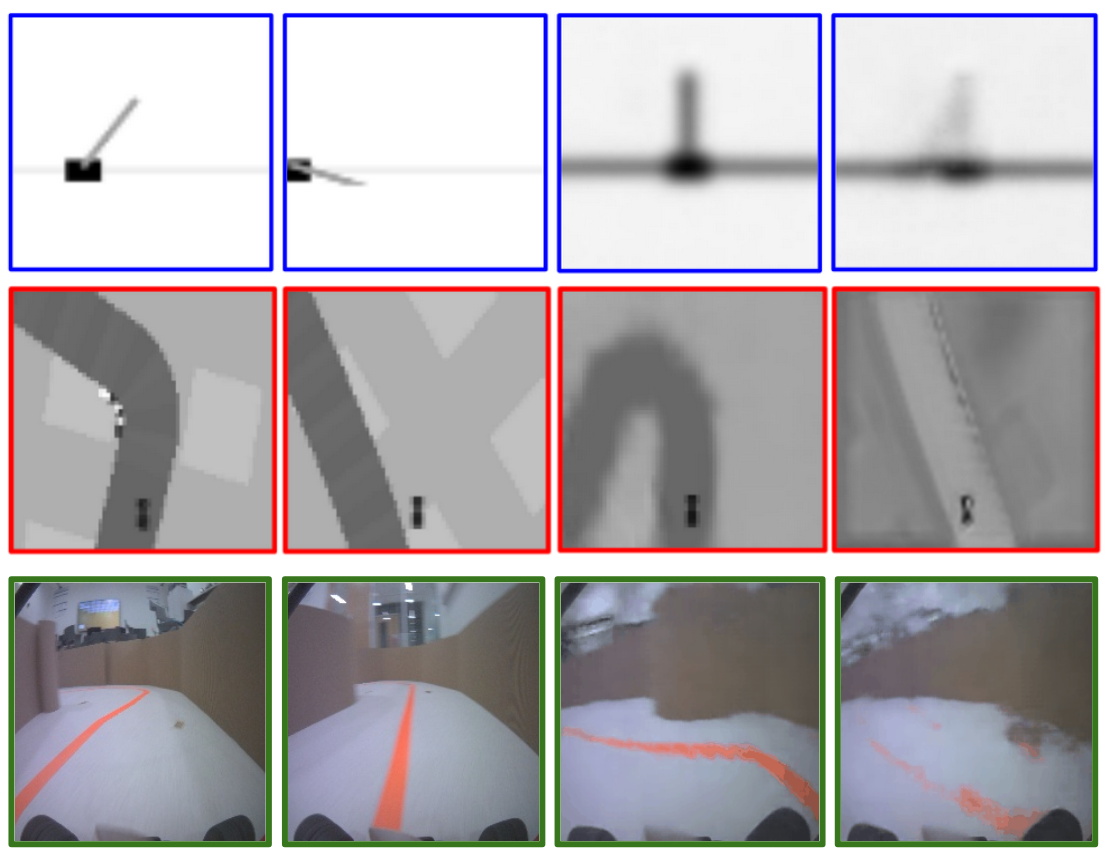}
    \caption{From top to bottom are the observations of our three cases (top to bottom: cart pole, racing car, and donkey car). The columns (left to right) are safe observation from $\obsspace$, unsafe observation from $\obsspace$, safe observation from $\dec(\latforecaster(\enc(\obsspace)))$, unsafe observation from $\imgforecaster(\obsspace)$ under distribution shift.
    }
    \label{fig:ood}
\end{figure}

The forecasts of latent vectors and observations in a composite pipeline are likely to be distribution-shifted, illustrated in Figure~\ref{fig:ood}. Such shifts degrade the accuracy of the safety evaluator. Since safety-labeling a significant amount of such data is impractical, we employ an \emph{unsupervised domain adaptation} strategy at test time to mitigate the impact on the evaluator accuracy without \textit{any} additional labels.

We adopt the \emph{Marginal Entropy Minimization} (MEMO) approach for per-sample test-time adaptation~\cite[]{zhang2022memo}. The key idea is to encourage the evaluator to produce confident and consistent predictions across multiple stochastically augmented versions of the same observation. This is achieved by: (1) generating augmented versions of the input visual observations; (2) aggregating their predictions into a marginal probability distribution; and (3) minimizing the entropy of this marginal distribution by updating the evaluator parameters.

Let
\[
\mathcal{A} = \{a_1, \ldots, a_B\}
\]
be a set of stochastic augmentation functions (e.g., random crop, color jitter, Gaussian noise), each mapping an image $y$ to an augmented image $a_i(y)$. For each class $s \in \{0, 1\}$, where $0$ denotes \emph{unsafe} and $1$ denotes \emph{safe}, we compute the \emph{marginal predicted probability}:
\begin{equation}
\bar{p}_\theta(s \mid y) = \frac{1}{B} \sum_{i=1}^{B} V_\theta(s \mid a_i(y)).
\end{equation}

The MEMO loss is the negative entropy of the marginal distribution:
\begin{equation}
L_{\mathrm{MEMO}}(\theta; y) = - \sum_{s \in \{0, 1\}} \bar{p}_\theta(s \mid y) \log \bar{p}_\theta(s \mid y),
\end{equation}
where $L_{\mathrm{MEMO}}$ encourages the evaluator to make low-entropy, consistent predictions across all augmented views.

The evaluator parameters are adapted with a single gradient descent step:
\begin{equation}
\theta^{+} \leftarrow \theta - \eta \cdot \nabla_\theta L_{\mathrm{MEMO}}(\theta; y),
\end{equation}
where $\eta > 0$ is the adaptation learning rate.

The adapted evaluator $V_{\theta^{+}}$ then outputs the predicted safety label:
\begin{equation}
\hat{s} = \arg\max_{s \in \{0,1\}} V_{\theta^{+}}(s \mid y).
\end{equation}

As indicated in Figures~\ref{fig:overview} and~\ref{fig:comp}, this label-free, per-sample adaptation improves evaluator robustness under latent distribution shift in composite latent predictors.

It is important to distinguish between two sources of distribution shift in our setting. First, the decoded images produced by the VAE may themselves be out-of-distribution for the safety evaluator, even when the underlying trajectory is in-distribution. This shift occurs because reconstruction artifacts and latent-space forecasting errors introduce visual distortions absent from the original training images. Second, the trajectories themselves may be out-of-distribution, representing scenarios not encountered during training. Our test-time adaptation strategy addresses the first source directly: by encouraging consistent, confident predictions across stochastically augmented views of each decoded image, MEMO makes the evaluator robust to the visual perturbations introduced by the decode and predict pipeline, without requiring any additional safety labels.

To mitigate the impact of distribution shifts occurring at the trajectory level, our framework currently relies on the conformal calibration mechanism described in the following section. When the system encounters a trajectory that significantly deviates from the training distribution, the predicted safety chance $P'(\varphi(x_{i+k})|y_i)$ naturally reflects higher uncertainty, leading to wider conformal intervals that alert the controller to potential unreliability~\cite{yang2024generalized}. 

This robustness can be further augmented by leveraging pre-trained foundation world models and multimodal meaning representations \cite{mao2024zero, liu2023meaning}. By utilizing the superior zero-shot generalization of these large-scale models, the safety evaluator can potentially reason about novel maneuvers or environmental geometries not present in the local training set, thereby bridging the gap between training and deployment distributions.

\subsection{Conformal Calibration for Chance Predictors}\label{subsec:concali}

\looseness=-1
Standard calibration methods operate on binary outcomes (safe/unsafe), which makes it difficult to obtain reliable probability guarantees for rare or noisy events. In our setting, we require guarantees not on single predictions but on average safety probabilities within bins, because the chance predictor outputs a continuous safety probability rather than a binary label. This necessitates bin averaging to obtain meaningful statistics, and motivates our use of adaptive binning combined with conformal prediction.


To turn a label predictor $\rho$ into a chance predictor $g$, we take its normalized softmax outputs and perform post‐hoc calibration~\cite[]{guo_calibration_2017,zhang_mix-n-match_2020}. On a held‐out calibration dataset $Z_{cal}$, we hyperparameter‐tune over state‐of‐the‐art post‐hoc calibration techniques: temperature scaling, logistic calibration, beta calibration, histogram binning, isotonic regression, ensemble of near‐isotonic regression (ENIR), and Bayesian binning into quantiles (BBQ)~\cite[]{wang2023calibration}.

To evaluate the quality of probability estimates produced by our chance predictors,
we consider two standard calibration metrics: \emph{Expected Calibration Error} (ECE)
and the \emph{Brier Score}~\cite[]{glenn1950verification,guo_calibration_2017}.
Let $\hat{p}_i \in [0,1]$ denote the predicted probability of safety for sample $i$,
and $y_i \in \{0,1\}$ the corresponding ground truth label.
Given $n$ samples and a binning function $\mathrm{bin}(\cdot)$ that assigns predictions to $Q$ disjoint bins,
ECE is defined as
\begin{equation}
\mathrm{ECE} \;=\; \sum_{j=1}^Q \frac{|B_j|}{n}
\left|\, \mathrm{acc}(B_j) - \mathrm{conf}(B_j) \,\right|,
\end{equation}
where $B_j$ is the set of samples in bin $j$,
$\mathrm{acc}(B_j) = \frac{1}{|B_j|} \sum_{i \in B_j} \mathbf{1}\{y_i = 1\}$
is the empirical accuracy, and
$\mathrm{conf}(B_j) = \frac{1}{|B_j|} \sum_{i \in B_j} \hat{p}_i$
is the mean predicted confidence in the bin.

The Brier Score measures the mean squared error between predicted probabilities and actual outcomes:
\begin{equation}
\mathrm{Brier Score} \;=\; \frac{1}{n} \sum_{i=1}^n \left( \hat{p}_i - y_i \right)^2 .
\end{equation}
Lower values of either metric indicate better-calibrated predictions.

By selecting the model that has the minimum ECE, we obtain calibrated softmax values $g(y_{i})$. Furthermore, to ensure that samples are spread evenly across bins to support our statistical guarantees, we perform \textit{adaptive binning} as defined below to construct a validation dataset $Z_{val}$, on which we obtain our guarantees.

\begin{definition}[Adaptive binning]\label{def:adaptive-binning}
Given a dataset $Z$, split $Z$ into $Q$ bins $\{B_{j}\}_{j=1}^{Q}$ by a constant count of samples, $\lfloor |Z|/Q\rfloor$ each. The resulting $\{B_{j}\}_{j=1}^{Q}$ is a binned dataset. From each bin $B_{j}$, we draw $N$ i.i.d.\ samples with replacement $M$ times to get $\{B_{j1},\ldots,B_{jM}\}$. For each resampled bin $B_{ji}$, we calculate
\[
\begin{split}
g_{ji} &= \frac{1}{|B_{ji}|}\sum_{y\in B_{ji}} g(y),\\
p_{ji} &= \frac{1}{|B_{ji}|}\sum_{y\in B_{ji}} \mathbf{1}\{\text{prediction is safe}\},\\
\delta_{ji} &:= \bigl|\,g_{ji} - p_{ji}\bigr|.
\end{split}
\]

Given a bin number $j$, our goal is to build prediction intervals $[0,c_{j}]$ that contain the calibration error of the next (unknown) average confidence $g_{j^{*}}$ that falls into bin $B_{j}$ with chance at least $1 - \alpha$:
\begin{equation}\label{eq:4}
P\bigl(\lvert g_{j^{*}} - p_{j}\rvert \le c_{j}\bigr) \;\ge\; 1 - \alpha
\quad \text{for } j = 1, \dots, Q,
\end{equation}
where $\alpha$ is the miscoverage level (a hyperparameter).
\end{definition}

That is, we aim to find a statistical upper bound $c_{j}$ on the error of our chance predictor in each bin. After we obtain a chance prediction $g_{j^{*}}$, it will be turned into an \textit{uncertainty‐aware interval}:\footnote{This is an interval on confidence, rather than a confidence interval.}
\[
\bigl[g_{j^{*}} - c_{j},\,g_{j^{*}} + c_{j}\bigr],
\]
which contains the true probability of safety in $1-\alpha$ cases. Notice that this guarantee is relative to the binned dataset fixed in Definition~\ref{def:adaptive-binning}.

To provide this guarantee, we apply conformal prediction in Algorithm~\ref{alg:concali}~\citep{lei2018distribution}. Intuitively, we rank the safety chance errors in resampled bins and obtain a statistical upper bound on this error for each bin. Similarly to existing works relying on conformal prediction, this algorithm guarantees equation \eqref{eq:4}~\cite[]{lindemann_conformal_2023,fan2020statistical}.  Recall that each calibration example $(\mathbf{y}_i, \safeprop(x_{i+k}))$ is drawn from an independent trajectory, so the calibration set satisfies the i.i.d.\ assumption required by conformal prediction.


\begin{theorem}[Conformal Bounds on Confidence Scores (adaptation of Theorem 2.1 in Lei et al. 2018\citep{lei2018distribution})]
Given a dataset bin $B = \{b_{k}\}_{k=1}^{K}$ of i.i.d.\ observation–state pairs $b_{k} = (y_{k}, x_{k})$, we obtain a collection of datasets $\{B_{j}\}_{j=1}^{M}$ by drawing $M$ datasets of $N$ i.i.d.\ samples from $B$, leading to datasets $B_{j}$ to be drawn i.i.d.\ from a dataset distribution $\mathcal{D}$. Then for another, unseen dataset $B_{M+1} \sim \mathcal{D}$, safety chance predictor $g$, and miscoverage level $\alpha$, calculating
\[
c \;=\; \operatorname{ConCali}(B,\,g,\,\alpha)
\]
leads to prediction intervals with guaranteed containment:
\[
P_{\mathcal{D}}\bigl(\lvert q(B_{M+1}) - p(B_{M+1})\rvert \le c\bigr) \;\ge\; 1 - \alpha,
\]
where $q(B_{M+1})$ is the mean safety chance prediction in $B_{M+1}$:
\[
q(B_{M+1}) \;\leftarrow\; \frac{1}{N}\sum_{l=1}^{N} g(y_{l}),
\quad \text{for each } y_{l} \in B_{M+1}.
\]

and $p(B_{M+1})$ is the mean true safety chance in $B_{M+1}$:
\[
p(B_{M+1}) \;\leftarrow\; \frac{1}{N} \sum_{l=1}^{N} \phi(x_{l}),
\quad \text{for each }\phi(x_{l}) \in B_{M+1},
\]
where $\phi(x)$ is the safety of $x$.
\end{theorem}

\begin{proof}
The proof follows the same structure as Theorem~2.1 in Lei et al. 2018, with the modification that non-conformity scores are computed on average safety probabilities within bins\citep{lei2018distribution}. Given that resampled bins $B_j$ are drawn i.i.d.\ from $\mathcal{D}$, the exchangeability assumption required by conformal prediction still holds. Therefore, ranking the absolute deviations $\delta_j$ yields the $(1-\alpha)$ coverage guarantee in \eqref{eq:4}.
\end{proof}

\begin{figure}[t]
  \centering
  \includegraphics[width=0.78\columnwidth]{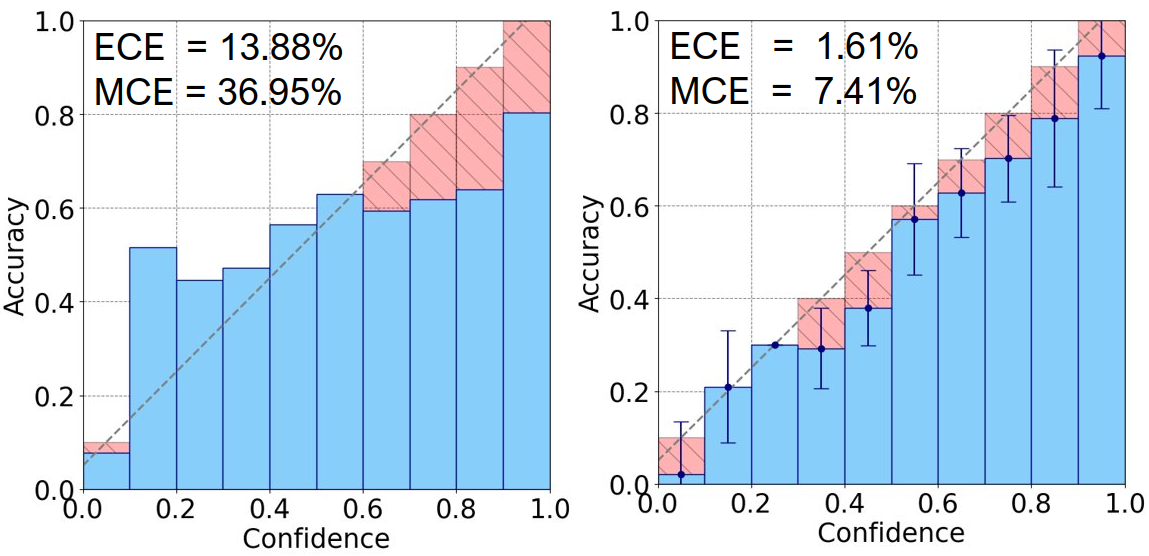}
  \caption{Calibration of a monolithic CNN predictor for the racing car with horizon $k=100$. Left: uncalibrated, right: calibrated via isotonic regression and conformal bounds for $\alpha=0.05$.}
  \label{fig:reliability-diagrams}
\end{figure}


\section{Experimental Evaluation}\label{sec:results}

\paragraph{Benchmark Systems}
    The racing car and cart pole from the OpenAI Gym and the realistic donkey car are selected as our study cases~\cite{brockman_openai_2016}. These environments are more suitable than pre-collected robotic datasets, such as the Waymo Open Dataset, for two reasons. First, to study safety prediction, we require a large dataset of safety violations, which is rarely found in real-world data. Second, the simulated dynamical systems are convenient for collecting an unbounded amount of data with direct access to the images and the ground truth for experimental evaluation (e.g., the car's position).

We defined the racing car's safety as being located within the track surface. The cart pole's safety is defined by the angles in the safe range of $[-6, 6]$ degrees, whereas the whole activity range of the cart pole is $[-48, 48]$ degrees. For the donkey car driving on a two-lane track, staying in the right lane is safe, whereas crossing over the lane boundary is unsafe.

\paragraph{Performance Metrics}

We use the \emph{F1 score} as our main metric for evaluating the accuracy of safety label predictors, as it balances the precision and recall:
\[
F1 = 2 \times \frac{\text{Precision} \times \text{Recall}}{\text{Precision} + \text{Recall}}
   = 2 \times \frac{TP}{TP + \frac{FP + FN}{2}}.
\]
False positives are also a major concern in safety prediction (actually unsafe situations predicted as safe), so we also evaluate the \emph{False Positive Rate} (FPR):
\[
\text{FPR} = \frac{FP}{FP + TN}.
\]
To evaluate our chance predictors, we compute the \emph{Expected Calibration Error} (ECE, intuitively a weighted average difference between the predicted and true probabilities) and the \emph{Brier Score} (mean squared error between predicted probabilities and actual binary outcomes)~\cite{guo_calibration_2017,minderer_revisiting_2021,glenn1950verification}.

\begin{figure*}[t]
  \centering
  \includegraphics[width=\textwidth]{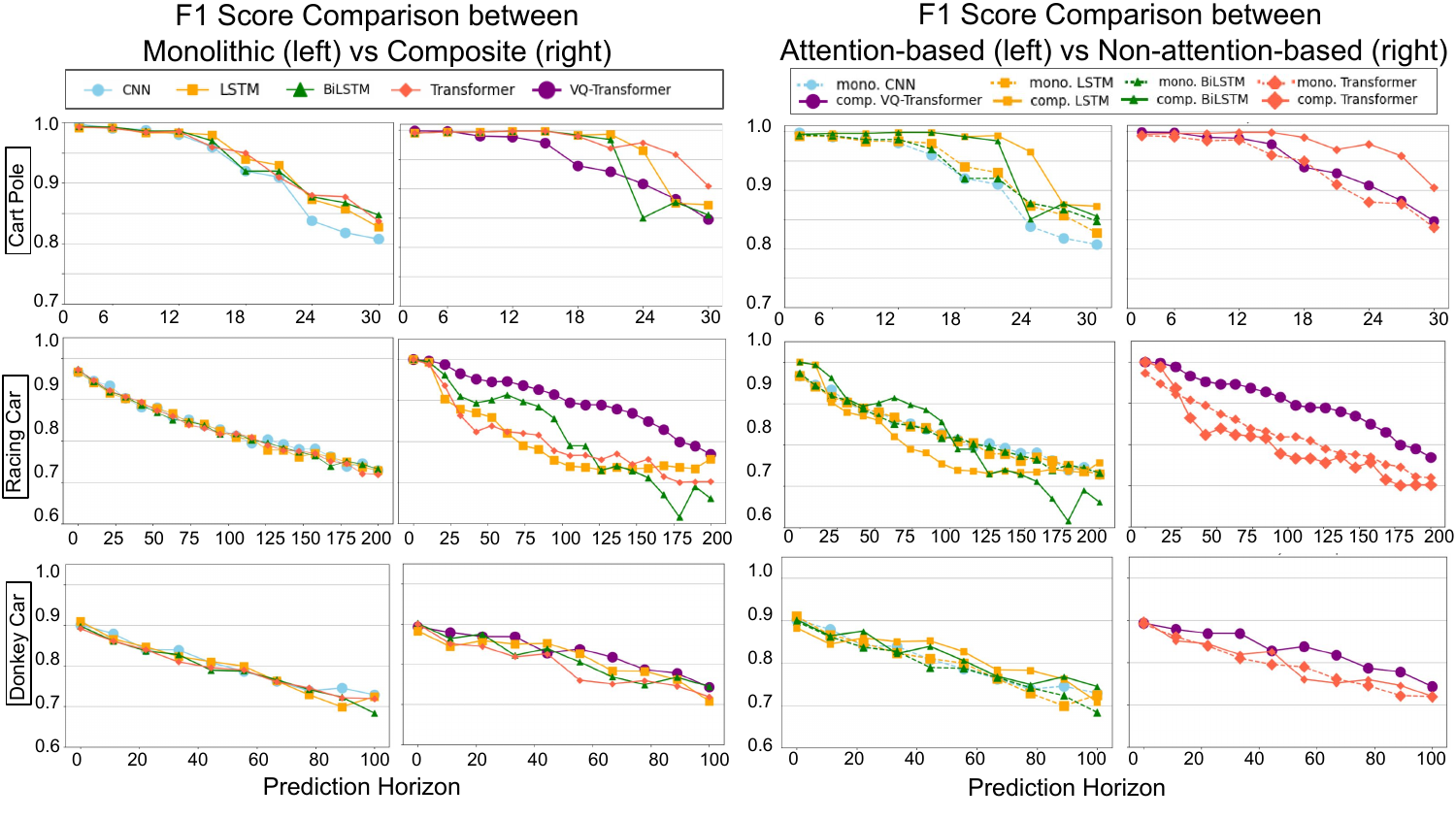}
  \caption{F1 score performance of safety label predictors over varied horizons. Upper to Lower: (1) cart pole; (2) car racing; (3) donkey car. Left pairs show the comparison between monolithic (mon.) models and composite (comp.) models, while right pairs show the comparison between attention-based models and non-attention-based models.}
  \label{fig:results}
\end{figure*}

\subsection{Experimental Setup}

\paragraph{Hardware}
The computationally heavy training was performed on a single NVIDIA A100 GPU, and other light tasks like data collection and conformal calibration were done on a CPU workstation with a 12th Gen Intel Core i9-12900H CPU and 64GB RAM in Ubuntu 22.04.

\paragraph{Datasets}
Deep Q-networks (DQN) were used to implement image-based controllers both for the racing car and the cart pole. For the racing car, we collected 240K samples for training and 60K for calibration/validation/testing each. For the cart pole, we collected 90K samples for each of the four datasets. All images were processed with normalization, which facilitated stable model training and improved generalization across different environments.

Each case study's data is randomly partitioned into four datasets. All the training tasks for the predictors are performed on the training dataset $\dataset_tr$. The calibration dataset $\dataset_cal$ is to tune predictor hyperparameters and fit the calibrators. The validation dataset $\dataset_val$ is used to produce conformal guarantees, and finally, the test dataset $\dataset_te$ is used to compute the performance metrics like F1, FPR, ECE, and Brier scores.

\subsubsection{Training details}
Pytorch 1.13.1 with the Adam optimizer was used for training. The maximum training epoch is 500 for VAEs and 100 for predictors.
The safety loss in the Equation.~\ref{eq:2} uses $\lambda_1 = 1$ and $\lambda_2 = 4096$, which equals the total pixel count in our images. The miscoverage level is $\alpha=0.05$. We used the Adam optimizer (initial learning rate $10^{-3}$, reduced by a factor of 0.1 with a patience of 5 epochs) and early stopping with a patience of 15 epochs. The maximum training epochs were 500 for VAEs and 100 for predictors. Batch sizes were 128 for classifiers, 64 for latent predictors, and 16 for image predictors. The CNN-based evaluator and single-image predictor used two convolutional layers (kernels 3 and 5) with max-pooling, followed by three linear layers and a softmax. The VAE encoder had four convolutional layers and two linear layers (latent size 32), and the decoder mirrored this with transposed convolutions. Sequence models used one LSTM layer and one linear layer, with monolithic predictors outputting 2 classes and composite latent forecasters outputting 32-dimensional vectors.
For each test sample, multiple augmented views are generated using seven types of transformations: autocontrast, equalize, rotation, solarization, shear, translation, and posterization in our UDA. These augmentations alter low-level visual attributes while preserving the semantic content of the original image. The model’s predictions over these views are aggregated, and the network parameters are adapted to improve stability and robustness against domain variations for our evaluator.
We released our code at \url{https://github.com/Trustworthy-Engineered-Autonomy-Lab/hsai-predictor}.

\subsection{Comparative Results}

Below, we present comparisons and ablations along the six key dimensions of the proposed family of predictors and also discuss our conformal calibration's performance. In the figures and tables, ``mon'' and ``comp'' stand for monolithic and composite, respectively.

\begin{table*}[!t]
  \centering
  \caption{F1 Score and FPR.}
  \label{tab:f1}
  \scriptsize
  \renewcommand{\arraystretch}{1.2}
  \setlength{\tabcolsep}{4pt}
\makebox[\textwidth][c]{%
\begin{tabular}{lccccccccc}

    \toprule
    \multirow{2}{*}{\textbf{Model}} &
    \multicolumn{3}{c}{\textbf{Racing Car}} &
    \multicolumn{3}{c}{\textbf{Cart Pole}} &
    \multicolumn{3}{c}{\textbf{Donkey Car}} \\
    \cmidrule(lr){2-4} \cmidrule(lr){5-7} \cmidrule(lr){8-10}
    & \textbf{k=60} & \textbf{k=120} & \textbf{k=180}
    & \textbf{k=10} & \textbf{k=20} & \textbf{k=30}
    & \textbf{k=30} & \textbf{k=60} & \textbf{k=90} \\
    \midrule
    \textbf{F1 score} & & & & & & & & & \\
    mono. CNN & 88.04 & 79.51 & 73.83 & \textbf{99.65} & 94.45 & 95.21 & 88.99 & 83.73 & \textbf{77.18} \\
    mono. LSTM & 87.89 & 80.48 & 74.95 & 98.29 & 93.97 & 85.72 & 84.69 & 79.92 & 69.89 \\
    mono. BiLSTM & 86.88 & 80.24 & 75.09 & 98.59 & 91.35 & 86.76 & 83.69 & 78.72 & 72.27 \\
    mono. Transformer & 87.34 & 80.89 & 74.55 & 98.38 & 94.45 & 87.21 & \textbf{88.99} & 77.83 & 73.25 \\
    comp. LSTM & 85.99 & 73.67 & 73.60 & \textbf{99.65} & \textbf{99.11} & 87.50 & 85.87 & 82.66 & 76.19 \\
    comp. BiLSTM & 90.17 & 78.93 & 61.54 & \textbf{99.65} & 98.37 & 87.02 & 87.49 & 80.55 & 76.73 \\
    comp. Transformer & 83.83 & 76.62 & 70.18 & 99.47 & 98.93 & \textbf{95.21} & 84.46 & 76.03 & 74.59 \\
    comp. VQ-Transformer & \textbf{94.55} & \textbf{89.35} & \textbf{79.03} & 98.65 & 93.83 & 88.21 & 88.18 & \textbf{83.73} & \textbf{77.18} \\
    \midrule
    \textbf{FPR} & & & & & & & & & \\
    mono. CNN & 27.25 & 39.48 & 54.11 & 8.41 & 28.43 & 34.39 & 26.35 & 37.16 & 42.83 \\
    mono. LSTM & 26.57 & 35.30 & 44.56 & 5.31 & 24.12 & 36.44 & \textbf{24.89} & 35.16 & 38.66 \\
    mono. BiLSTM & 26.69 & 36.35 & 45.34 & 10.36 & 22.27 & 35.93 & 27.47 & 33.63 & 47.54 \\
    mono. Transformer & 25.48 & 36.81 & 42.80 & 13.86 & 34.08 & 37.29 & 25.76 & 31.29 & 40.33 \\
    comp. LSTM & \textbf{16.75} & \textbf{23.98} & 35.41 & 10.38 & 19.68 & 39.45 & 32.56 & 30.46 & \textbf{36.71} \\
    comp. BiLSTM & 28.57 & 30.71 & \textbf{34.29} & \textbf{3.09} & \textbf{16.92} & 46.33 & 28.87 & \textbf{27.23} & 40.37 \\
    comp. Transformer & 17.13 & 26.56 & 35.71 & 14.51 & 30.91 & 31.34 & 41.10 & 30.36 & 47.90 \\
    comp. VQ-Transformer & 18.18 & 27.01 & 35.72 & 3.56 & 24.94 & \textbf{28.28} & 32.12 & 34.73 & 41.92 \\
    \bottomrule
  \end{tabular}
  }
\end{table*}

\begin{table*}[t]
  \centering
  \caption{ECE and Brier Score before (gray rows) and after (white rows) calibration.}
  \label{tab:ece}
  \scriptsize
  \renewcommand{\arraystretch}{1.2}
  \resizebox{\textwidth}{!}{%
  \begin{tabular}{lccc|ccc|ccc}
    \toprule
    \textbf{Model} & \multicolumn{3}{c|}{\textbf{Racing Car}} & \multicolumn{3}{c|}{\textbf{Cart Pole}} & \multicolumn{3}{c}{\textbf{Donkey Car}} \\
    & \textbf{k=60} & \textbf{k=120} & \textbf{k=180}
    & \textbf{k=10} & \textbf{k=20} & \textbf{k=30}
    & \textbf{k=30} & \textbf{k=60} & \textbf{k=90} \\
    \midrule
    \textbf{ECE} \\
          
 \rowcolor{gray!20}
mono. CNN& 0.0715& 0.1349& 0.1390  
& 0.1348
& 0.4031
& 0.7393
& 0.0664
& 0.1807
& 0.2798
\\
      mono. CNN &0.0079& 0.0094& 0.0090
& 0.0031
& 0.0030
& \textbf{0.0057}
& 0.0292
& 0.0614
& 0.0745

      \\

       \rowcolor{gray!20}
        mono. LSTM &0.1197& 0.1868& 0.1675 
& 0.0014
& 0.0096
& 0.2121
& 0.1447
& 0.2102
& 0.3336
        \\

        mono. LSTM &0.0054 & 0.0100&0.0119
& \textbf{0.0007}
& 0.0038
& 0.0176
& 0.0958
& 0.1971
& 0.2075
        \\
 \rowcolor{gray!20}        
mono. BiLSTM & 0.1558
& 0.2506
& 0.1905
& 0.0038
& 0.0018
& 0.1483
& 0.1702
& 0.2646
& 0.2705
\\
mono. BiLSTM
 & 0.0133
& 0.0149
& 0.0467
& 0.0029
& \textbf{0.0006}
& 0.0127
  & 0.1286
  & 0.1521
  & 0.2354

 \\
 
  \rowcolor{gray!20}      
  
  mono. Transformer& 0.1019
& 0.1121
& 0.3011
& 0.7341
& 0.4201
& 0.6722
& 0.0897
& 0.1751
& 0.2689

 \\mono. Transformer& 0.0164
& 0.0095
& 0.0194
& 0.0008
& 0.0023
& 0.0164
& 0.0667
& 0.0704 
& 0.0875

 \\

        \rowcolor{gray!20}
        comp. LSTM & 0.0052 & 0.0012& 0.0039 
& 0.1069
& 0.4659
& 0.5264
& 0.1112

      & 0.1304  
      & 0.2425
        \\

            comp. LSTM &\textbf{0.0002}&\textbf{0.0002}& 0.0007& 0.0045& 0.0088& 0.0119 
            & 0.0405

            & 0.0864
& \textbf{0.0320}
            \\

                   \rowcolor{gray!20}
 
               comp. BiLSTM & 0.2165
& 0.0725
& 0.1646 
& 0.0914
& 0.3388
& 0.4635

& 0.1171
& 0.1727
& 0.2108

\\
 
               comp. BiLSTM & 0.0093
& 0.0159
& 0.0215 
& 0.0834
& 0.1086
& 0.1419

& 0.0249
& 0.0809

& 0.1280

\\

                   \rowcolor{gray!20}

        comp. Transformer &   0.1803
& 0.3511
& 0.5474
& 0.1437
& 0.0281
& 0.0826
& 0.0851
& 0.1215
& 0.2034

\\

        comp. Transformer & 0.0074
& 0.0005
& 0.0143
& 0.0051
& 0.0255
& 0.0107

& 0.0769
& \textbf{0.0084}
& 0.0589

\\
            \rowcolor{gray!20}
      comp. VQ-Transformer& 0.0144
& 0.0026
& 0.0515
& 0.0100
& 0.1426
& 0.2589
& 0.0661
& 0.1229
& 0.2209

\\

      comp. VQ-Transformer& 0.0008
& 0.0032
& \textbf{0.0002}
& 0.0026
& 0.0082
& 0.0111
& \textbf{0.0038}
& 0.0739
& 0.1455

\\
        \midrule
       \textbf{Brier Score}\\

                     \rowcolor{gray!20}

      mono. CNN & 0.2269
& 0.2708
& 0.2400
& 0.0022
& 0.0405
& 0.1895
& 0.1298
& 0.2452
& 0.2680
\\
      mono. CNN & 0.1306
& 0.1067
& 0.1269
& 0.0018
& \textbf{0.0136}
& 0.0720
& 0.1075
& 0.1287
& 0.2162
\\
              \rowcolor{gray!20}

        mono. LSTM&  0.1965
& 0.3317
& 0.3079
& 0.0017
& 0.0275
& 0.1521
& 0.1547
& 0.2486
& 0.2834
        \\
 mono. LSTM 
 & \textbf{0.0020}
& 0.0993
& 0.1804
& \textbf{0.0011}
& 0.0180
& 0.1365
& 0.0095
& 0.0527
& 0.1644
        \\
                      \rowcolor{gray!20}
mono. BiLSTM 
& 0.1083
& 0.4916
& 0.2017
& 0.2822
& 0.5659
& 0.3453
& 0.0399
& 0.1038
& 0.1902
        \\
mono. BiLSTM 
& 0.0796
& 0.2248
& 0.1619
 & 0.0039
& 0.0184
& 0.1518
& 0.0042
& 0.0186
& \textbf{0.0728}
 \\
              \rowcolor{gray!20}
              mono. Transformer
& 0.0126
& 0.2556
& 0.1777
& 0.4970
& 0.9131
& 0.6400
& 0.0734
& 0.0971
& 0.1611
 \\

             mono. Transformer

& 0.0235
& 0.0523
&0.0939
& 0.0197
& 0.0159
& \textbf{0.0382}
& 0.0697
& 0.0797
& 0.1223
       \\            
       \rowcolor{gray!20}
        comp. LSTM & 0.1938
& 0.2663
& 0.3749
& 0.1975
& 0.1281
& 0.3551
& 0.1197
& 0.1214
& 0.1568

\\

        comp. LSTM& 0.1666
& 0.2318
& 0.2635
& 0.0519
& 0.0533
& 0.1045
& 0.0866
& 0.0957
& 0.1383
\\
       \rowcolor{gray!20}

               comp. BiLSTM & 0.2742
& 0.3393
& 0.3688
& 0.0915
& 0.2025
& 0.3337
& 0.0986
& 0.1117
& 0.1574
\\
               comp. BiLSTM 
               
              & 0.1343
& 0.2026
& 0.1937
& 0.0280
& 0.0211
& 0.1255
& 0.0836
& 0.0932
& 0.1501
\\
       \rowcolor{gray!20}

        comp. Transformer &  0.1152
& 0.1491
& 0.2169
& 0.119
& 0.2474
& 0.2782
& 0.0648
& 0.0813
& 0.0903

\\
        comp. Transformer & 0.0645
& 0.0981
& 0.1150
& 0.0638
& 0.0812
& 0.0987
& \textbf{0.0222}
& 0.0774
& 0.0786
\\
       \rowcolor{gray!20}

      comp. VQ-Transformer
      & 0.0674
& 0.2380
& 0.2587
& 0.2035
& 0.2584
& 0.2411
& 0.0792
& 0.0887
& 0.0975
\\
      comp. VQ-Transformer
      &0.0600
& \textbf{0.0559}
& \textbf{0.0521}

      & 0.0554
& 0.0617
& 0.0774
& 0.0544
& \textbf{0.0616}
& 0.0951

       \\
    \bottomrule
  \end{tabular}
    }
\end{table*}

\paragraph{Composite predictors outperform monolithic ones}
Monolithic predictors use their full learning capacity to predict safety but do not learn the underlying dynamics. Therefore, our initial hypothesis was that they would do well on short horizons --- but lose to composite predictors on longer horizons. The results support this hypothesis, as illustrated in Figure~\ref{fig:results}: the performance of monolithic predictors degrades faster as the horizon grows. Moreover, composite predictors offered a better FPR as shown in Table~\ref{tab:f1} for the donkey cars, which may be desirable in safety-critical systems. We contend that the inferior performance of monolithic predictors arises from the inherent difficulty of learning coherent, long-term latent dynamics --- a challenge that remains largely unresolved. By contrast, composite predictors achieve superior results by explicitly decomposing the safety‐prediction problem into three stages: (1) learning a suitable latent representation; (2) forecasting future states, and (3) evaluating those states’ safety. This architectural separation imposes a clear inductive bias: the model is first guided to generate plausible trajectories and then to judge their safety, rather than expecting a single network to discover both the underlying dynamics and the safety criterion simultaneously. As a result, each subtask becomes substantially more tractable, leading to higher overall performance.

\begin{table}[!t]
\caption{F1 scores of evaluators on different datasets (in \%).}
\label{tab:ev1}
\centering
\footnotesize
\begin{tabular}{llccc}
\toprule
\textbf{Case Study} & \textbf{Evaluator Type} & \textbf{Normal} & \textbf{OOD} & \textbf{Mix} \\
\midrule
\multirow{3}{*}{Racing Car}
& Image Evaluator       & \textbf{ 99.56} & 63.94 & 78.23 \\
& Latent Evaluator      & 98.79 & 65.73 & 80.27 \\
& UDA Image Evaluator   & 97.13 & \textbf{81.65} & \textbf{91.34  } \\
\midrule
\multirow{3}{*}{Cart Pole}
& Image Evaluator       & 99.66 & 67.21 & 91.05 \\
& Latent Evaluator      & \textbf{99.67} & 69.14 & 89.33 \\
& UDA Image Evaluator   & 99.06 & \textbf{ 76.52} & \textbf{96.62} \\
\midrule
\multirow{3}{*}{Donkey Car}
& Image Evaluator       & 96.20 & 93.03 & 94.43 \\
& Latent Evaluator      & \textbf{ 97.19} & 92.21 & 94.70 \\
& UDA Image Evaluator   & 96.05 & \textbf{94.14} & \textbf{95.97} \\
\bottomrule
\end{tabular}
\end{table}

\begin{table}[!t]
\caption{FPR of evaluators on different datasets (in \%).}
\label{tab:ev2}
\centering
\footnotesize
\begin{tabular}{llccc}
\toprule
\textbf{Case Study} & \textbf{Evaluator Type} & \textbf{Normal} & \textbf{OOD} & \textbf{Mix} \\
\midrule
\multirow{3}{*}{Racing Car}
& Image Evaluator       & 0.42 & 52.78 & 23.53 \\
& Latent Evaluator      & 2.15 & 30.82 & 19.62 \\
& UDA Image Evaluator   & 0.62 & 7.07 & 4.37 \\
\midrule
\multirow{3}{*}{Cart Pole}
& Image Evaluator       & 0.53 & 21.51 & 9.88 \\
& Latent Evaluator      & 0.61 & 18.34 & 7.69 \\
& UDA Image Evaluator   & 1.34 & 16.72 & 10.03 \\
\midrule
\multirow{3}{*}{Donkey Car}
& Image Evaluator       & 2.69 & 18.25 & 6.05 \\
& Latent Evaluator      & 3.05 & 11.97 & 5.32 \\
& UDA Image Evaluator   & 3.21 & 9.15 & 5.42 \\
\bottomrule
\end{tabular}
\end{table}

\paragraph{Attention significantly improves composite predictor performance}
Incorporating attention blocks via transformers into the composite predictor architecture yielded consistent improvements in both FPR and F1 scores across all forecasting horizons as shown in Table~\ref{tab:f1}. This suggests that leveraging both past and future context enhances the model's ability to produce temporally coherent and accurate predictions. Additionally, we observed a substantial decrease in ECE as indicated in Table~\ref{tab:ece}, indicating that the VQ-transformers not only perform better but also make significantly more calibrated predictions.

\paragraph{Latent evaluators outperform image-based ones}
Latent evaluators outperform image-based ones both in F1 and FPR, as shown in the Tables~\ref{tab:ev1} and~\ref{tab:ev2}.
This is primarily because the VAE’s reconstruction objective encourages the latent space to capture fine-grained image details, while the prior imposes additional structure that facilitates classification. Moreover, KL regularization promotes smoothness of the latent manifold and can implicitly disentangle task-relevant features.
Additionally, the latent representation filters out pixel-level noise and, by mapping sequential observations into a compact and structured manifold, preserves the temporal relationships between states. This smoothness in latent trajectories makes it easier for the evaluator to capture long-term safety-relevant dynamics and avoid being misled by transient visual variations.

\paragraph{UDA evaluator improves robustness under distribution shift}
We evaluate the effectiveness of unsupervised domain adaptation in enhancing the evaluator's robustness to distribution shift. In addition to the normal images, we test on two special test subsets: \textit{OOD}, which consists of 500 images sampled from predictions with manual labels that exhibit strong distribution shift, and \textit{Mix}, which randomly selects 500 predictions containing both in-distribution and OOD examples, with proportions varying by case study. As shown in Tables~\ref{tab:ev1} and \ref{tab:ev2}, across all case studies, UDA-based image evaluators consistently outperform both standard image and latent evaluators on the OOD and Mix subsets, without losing almost any performance on the normal data. These results demonstrate that UDA is a promising technique to generalize safety evaluators to challenging out-of-distribution scenarios.

\paragraph{Calibration improves safety‐chance predictor reliability.}
The reliability diagram in Figure~\ref{fig:reliability-diagrams} exemplifies a common situation among uncalibrated chance predictors: underconfident for the rejected class (below $0.5$), overconfident for the chosen class (above $0.5$). The aggregate calibrated results for different architectures and prediction horizons are shown in Table~\ref{tab:ece}. We can see that our calibration consistently reduces the overconfidence and leads to a lower ECE, even for long prediction horizons. The most effective calibrators were the isotonic regression for the racing car and ENIR for the cart pole and donkey car. This leads us to conclude that safety-chance prediction is a more suitable problem formulation for highly uncertain image-controlled robots rather than safety-label prediction.

\paragraph{Conformal calibration coverage is reliable}
Supporting our theoretical claims, the predicted intervals for calibration errors contained the true error values from the test data (Figure~\ref{fig:coverage}). The coverage for the racing car is $96.24\%\pm4.75\%$, the cart pole is $96.09\%\pm5.71\%$, and the donkey car is $98.52\%\pm3.76\%$.
Our average error bound is $0.0924\pm0.0438$ for the racing car, $0.0409\pm0.0200$ for the cartpole, and $0.0628\pm0.0296$ for the donkey car. Therefore, our calibration intervals, constructed as predicted chance $\pm$ its bin's calibration bound, can be used reliably and informatively in online robotic tasks like prediction, planning, and adaptation.

\begin{figure*}[t]
  \centering
  \includegraphics[width=\textwidth]{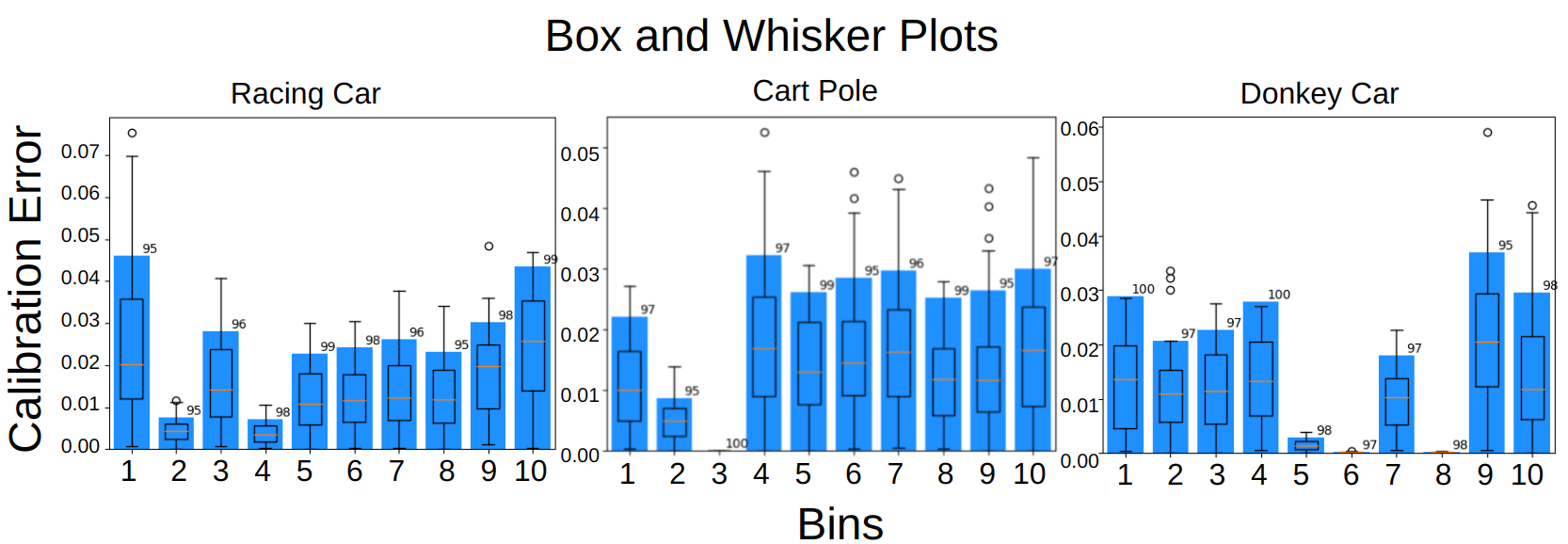}
  \caption{Our conformal bounds (in blue) for calibration error of the monolithic predictor (left to right: racing car, cart pole, and donkey car) with horizon $k=18$ for $\alpha=0.05$ contain more than 95\% of the true calibration errors (box and whisker plots).}
  \label{fig:coverage}
\end{figure*}

\section{Discussion and Conclusion}
\label{sec:conclusion}

\paragraph{Limitations} Our predictor family's scope is limited to systems where safety can be inferred from raw sensor data. This assumption holds in domains like autonomous driving, where an imminent collision, lane departure, or off-track event can be visually detected. Relatedly, the safety specifications considered in this work are restricted to simple state-based binary predicates ($\varphi: \mathcal{X} \rightarrow \{0,1\}$) and do not capture temporal requirements. 
 In contrast, there exist systems where safety depends on unobservable internal states or variables that are not visually apparent. For example, for an aircraft whose structural integrity depends on internal component stresses, raw visual or range data would be insufficient to assess safety.
 
Obtaining labeled safety data is associated with notable costs, especially for negative safety samples in physical systems. Our approach is limited to systems where it is possible to generate a large volume of labeled data to effectively train the predictors. For instance, our donkey car case study required labeling over 10,000 images, which took approximately 5-10 hours of human effort, including iterative refinement. While this labeling burden is non-trivial, it is feasible for many robotics applications where safety-critical scenarios can be systematically collected. Furthermore, this burden can be substantially alleviated through automated and semi-automated labeling techniques. Recent frameworks for active learning, weak supervision, and programmatic labeling have demonstrated significant reductions in manual annotation effort for similar vision-based tasks~\cite{settles2009active,ratner2017snorkel,ratner2019snorkel}. 

Some prediction approaches have performed poorly, like monolithic \emph{flexible-horizon} predictors, which output the time to the first expected safety violation. Their prediction space appeared too complex and insufficiently regularized, leading to poor learning and performance. Also, existing test-time adaptation methods operate at the image level, which is potentially inefficient. Exploring adaptation in latent representation could further improve the performance of the latent evaluator under distribution shift.
Thus, learning strong safety-informed representations for autonomy remains an open problem.
With the rise of large language models, meaning representations in multimodal autoregressive models may provide more possibilities to explore the surrogate dynamics of autonomy systems~\cite[]{liu2023meaning}. 

\paragraph{Future work}
\looseness=-1
Interesting future directions include adding physical constraints to latent states similar to Neural ODEs, using a generative model to obtain surrogate unsafe samples from a more realistic environment, and applying vision-language and vision-language-action models to automatically label the safety of generative predictions~\cite[]{wen_social_2022}. Another opportunity is to mitigate the labeling burden via transfer learning from simulation and generative models. Furthermore, our framework can be extended to support safety specifications in linear temporal logic (LTL), enabling safety predictions that reason over sequences of events and long-horizon behavioral constraints~\cite[]{hahn2021teaching}.
\backmatter

\section*{Declarations}

\subsection*{Acknowledgments}

First, the authors would like to thank Cade McGlothlin and Rohith Nama Reddy for experimenting with prediction architectures. Second, they would like to thank Srikruth Puram, Shaurya Wahal, and Pham ``Tony'' Hung for their work on various simulators. Finally, the authors also thank Dong-You ``Sam'' Jhong, Ao Wang, Chenxi Zhou, and Cesar Valentin for setting up Donkey Cars.

\subsection*{Funding}

The authors disclose receipt of the following financial support for the research, authorship, and publication of this article. This work was supported in part by the NSF Grant CNS 2513076. Any opinions, findings, or conclusions expressed in this material are those of the authors and do not necessarily reflect the views of the National Science Foundation (NSF) or the U.S. Government.


\bibliography{new,response,full-autogenerated}

\end{document}